\documentclass[journal]{IEEEtran}
\usepackage[utf8]{inputenc} 
\usepackage[T1]{fontenc}    
\usepackage{hyperref}       
\usepackage{url}            
\usepackage{booktabs}       
\usepackage{amsfonts}       
\usepackage{nicefrac}       
\usepackage{microtype}      
\usepackage{amsmath}
\usepackage{amsthm}
\usepackage{amsbsy}
\newtheorem{theorem}{Theorem}
\usepackage{amsfonts,amssymb}
\usepackage{stmaryrd}
\usepackage{xcolor}
\usepackage{graphicx}
\usepackage{subcaption}
\renewenvironment{proof}{{\bfseries Proof}}{\qed}

\newcommand{\R}{\mathbb{R}}
\newcommand{\Hilbert}{\mathcal{H}}
\newcommand{\K}{\mathcal{K}}

\newcommand{\dist}{\mathop{\mathrm{dist}}}
\newcommand{\vspan}{\mathop{\mathrm{span}}}
\newcommand{\minimize}[1]{\underset{#1}{\operatorname{minimize}}\,}
\newcommand{\maximize}[1]{\underset{#1}{\operatorname{maximize}}\,}

\begin{document}

\newcommand{\ab}{\mathbf{a}}
\newcommand{\bb}{\mathbf{b}}
\newcommand{\eb}{\mathbf{e}}
\newcommand{\ib}{\mathbf{i}}
\newcommand{\pb}{\mathbf{p}}
\newcommand{\qb}{\mathbf{q}}
\newcommand{\vb}{\mathbf{v}}
\newcommand{\ub}{\mathbf{u}}
\newcommand{\xb}{\mathbf{x}}
\newcommand{\yb}{\mathbf{y}}
\newcommand{\zb}{\mathbf{z}}

\newcommand{\Ab}{\mathbf{A}}
\newcommand{\Bb}{\mathbf{B}}
\newcommand{\Cb}{\mathbf{C}}
\newcommand{\Db}{\mathbf{D}}
\newcommand{\Eb}{\mathbf{E}}
\newcommand{\Fb}{\mathbf{F}}
\newcommand{\Gb}{\mathbf{G}}
\newcommand{\Hb}{\mathbf{H}}
\newcommand{\Ib}{\mathbf{I}}
\newcommand{\Jb}{\mathbf{J}}
\newcommand{\Kb}{\mathbf{K}}
\newcommand{\Lb}{\mathbf{L}}
\newcommand{\Pb}{\mathbf{P}}
\newcommand{\Qb}{\mathbf{Q}}
\newcommand{\Rb}{\mathbf{R}}
\newcommand{\Sb}{\mathbf{S}}
\newcommand{\Vb}{\mathbf{V}}
\newcommand{\Ub}{\mathbf{U}}
\newcommand{\Wb}{\mathbf{W}}
\newcommand{\Xb}{\mathbf{X}}
\newcommand{\Yb}{\mathbf{Y}}
\newcommand{\Zb}{\mathbf{Z}}

\newcommand{\alphab}{\boldsymbol{\alpha}}
\newcommand{\betab}{\boldsymbol{\beta}}
\newcommand{\gammab}{\boldsymbol{\gamma}}
\newcommand{\phib}{\boldsymbol{\phi}}
\newcommand{\Phib}{\boldsymbol{\Phi}}
\newcommand{\Qhib}{\boldsymbol{\Qhi}}
\newcommand{\omegab}{\boldsymbol{\omega}}
\newcommand{\psib}{\boldsymbol{\psi}}
\newcommand{\sigmab}{\boldsymbol{\sigma}}
\newcommand{\nub}{\boldsymbol{\nu}}
\newcommand{\thetab}{\boldsymbol{\theta}}
\newcommand{\delb}{\boldsymbol{\delta}}
\newcommand{\rhob}{\boldsymbol{\rho}}
\newcommand{\Pib}{\boldsymbol{\Pi}}
\newcommand{\pib}{\boldsymbol{\pi}}
\newcommand{\Sigmab}{\boldsymbol{\Sigma}}


\newcommand{\Cc}{\mathcal{C}}
\newcommand{\Ec}{\mathcal{E}}
\newcommand{\Fc}{\mathcal{F}}
\newcommand{\Hc}{\mathcal{H}}
\newcommand{\Kc}{\mathcal{K}}
\newcommand{\Lc}{\mathcal{L}}
\newcommand{\Nc}{\mathcal{N}}
\newcommand{\Oc}{\mathcal{O}}
\newcommand{\Pc}{\mathcal{P}}
\newcommand{\Rc}{\mathcal{R}}
\newcommand{\Uc}{\mathcal{U}}
\newcommand{\Xc}{\mathcal{X}}
\newcommand{\Yc}{\mathcal{Y}}
%
\title{\vspace{5mm} Learning from Non-Random Data in Hilbert Spaces:\\
An Optimal Recovery Perspective}
%
%
%

\author{Simon Foucart, Chunyang Liao, Shahin Shahrampour, and Yinsong Wang
\thanks{Simon Foucart and Chunyang Liao are with the Department of Mathematics, Texas A\&M University, College Station, TX 77843 USA.}
\thanks{Shahin Shahrampour and Yinsong Wang are with the Department of Industrial \& Systems Engineering, Texas A\&M University, College Station, TX 77843 USA.}
}

\maketitle

\begin{abstract}
The notion of generalization in classical Statistical Learning is often attached to the postulate that data points are independent and identically distributed (IID) random variables. While relevant in many applications, this postulate may not hold in general, encouraging the development of learning frameworks that are robust to non-IID data. In this work, we consider the regression problem from an Optimal Recovery perspective. Relying on a model assumption comparable to choosing a hypothesis class, a learner aims at minimizing the worst-case error, without recourse to any probabilistic assumption on the data. We first develop a semidefinite program for calculating the worst-case error of any recovery map in finite-dimensional Hilbert spaces. Then, for any Hilbert space, we show that Optimal Recovery provides a formula which is user-friendly from an algorithmic point-of-view, as long as the hypothesis class is linear. Interestingly, this formula coincides with kernel ridgeless regression in some cases, proving that minimizing the average error and worst-case error can yield the same solution. We provide numerical experiments in support of our theoretical findings. 
\end{abstract}



\ifCLASSOPTIONpeerreview
\begin{center} \bfseries EDICS Category: 3-BBND \end{center}
\fi
%
\IEEEpeerreviewmaketitle

\section{Introduction}
%
%
%
%
\IEEEPARstart{L}{et} us place ourselves in a classical scenario where data about an unknown function $f_0$
take the form 
\begin{equation}
\label{data}
y_i = f_0(\xb_i),
\qquad i \in [1:m].
\end{equation}
The values $y_i \in \R$ and the evaluations points $\xb_i \in \Omega \subseteq \R^d$ are available to the learner. 
The goal is to `learn' the function $f_0$ from the data \eqref{data} by producing a surrogate function $\hat{f}$ for $f_0$. 
Supervised Machine Learning methods compute such an $\hat{f}$ from a hypothesis class selected in advance. 
The performance of a method then depends on the choice of this hypothesis class: a good class should obviously approximate functions of interest well. 
This translates into a small {\em approximation error}, which is one of the constituents towards the total error of a method. 
Another constituent is the {\em estimation error}. 
In classical Statistical Learning \cite{vapnik1999overview}, the latter is often analyzed by adopting a postulate that the $\xb_i$'s are independent realizations of a random variable with an unknown distribution on $\Omega$. 
While relevant in many applications, this postulate may not hold in general, encouraging the development of learning frameworks that are robust to non-IID data.

In this work, we consider the regression problem from an Optimal Recovery perspective, without recourse to any probabilistic assumption on the data. 
Indeed, in the absence of randomness, an average-case analysis is not possible anymore. Instead, the learner aims at minimizing the worst-case (prediction) error by relying on a model assumption comparable to choosing a hypothesis class. 
We restrict our attention here to Hilbert spaces and provide the following contributions:
\begin{itemize}
    \item We develop a numerical framework for calculating the worst-case error in the case
    of finite-dimensional Hilbert spaces. In particular, we show that this error can be computed via a semidefinite program (Theorem \ref{computation_wce}). 
    \item We show that Optimal Recovery provides a formula which is user-friendly from an algorithmic point-of-view when the hypothesis class is a linear subspace (Theorem~\ref{Explicit_OR}). Interestingly, this formula coincides with kernel ridgeless regression in some cases (Theorem~\ref{arbitraryV}), proving that minimizing the average error and worst-case error can yield the same solution.
\end{itemize}
The theoretical findings
are verified through some numerical experiments presented in Section \ref{sec:simul}. 

\subsection{Why Optimal Recovery?} The theory of Optimal Recovery was developed in the 70's-80's as a subfield of Approximation Theory (see the surveys \cite{MicRiv,MicRiv2}). 
Its development was shaped by concurrent developments in the theory of spline functions (see e.g. \cite{de1963best,duchon1977splines}).
Splines provided a rare example where the theory integrated computations \cite{de1977computational}. 
But, at that time, algorithmic issues were not the high priority that they have become today and theoretical questions such as the existence of linear optimal algorithms prevailed (see e.g. the survey \cite{packel1988linear}). 
Arguably, this neglect hindered the development of the topic and this work can be seen as an attempt to promote an algorithmic framework that sheds light on similarities and differences between Optimal Recovery (in Hilbert spaces) and Statistical Learning. 
Incidentally, what is sometimes called the {\em spline algorithm} in Optimal Recovery has recently made a reappearance in Machine Learning circles as minimum-norm interpolation \cite{belkin2018understand,rakhlin2019consistency,Liang2018JustIK}, 
of course with a different motivation. 
We also remark that Optimal Recovery is not the only framework dealing with non-IID data. 
There are indeed other strands of Machine Learning literature (e.g. Online Learning \cite{hazan2016introduction} and Federated Learning \cite{zhao2018federated}) that investigate learning from non-IID and/or non-random data.

\subsection{Noisy observations.} 
A careful reader may wonder about the possibility of incorporating an error $e_i \in \R$ in the data $y_i = f_0(\xb_i)+e_i$, which is a common consideration in Machine Learning. 
We do not investigate such a scenario in this work, 
as our main focus is on 
drawing interesting connections between Optimal Recovery and some of the common Supervised Learning techniques
in the simplest of settings first. 
Future works will concentrate on this inaccurate scenario
which, despite some existing results (see \cite{plaskota1996noisy,ettehad2020instances,beck2007regularization}),
presents some unsuspected subtleties.
For instance, the results from \cite{beck2007regularization}
are only valid in the complex setting and not in the real setting considered here.

\section{The Optimal Recovery Perspective}
In this section, we recall the general framework of Optimal Recovery and highlight some novel results, including the computation of worst-case error and the explicit formula of optimal recovery map. 

\subsection{The function space.}
Echoing the theory of Optimal Recovery,
we consider the function $f_0$ more abstractly as an element from a normed space $\Fc$. 
The output data $y_i$'s, 
which are evaluations of $f_0$ at the points $\xb_i$'s, 
can be generalized to linear functionals $\ell_i$'s applied to $f_0$,
so that the data take the form
\begin{equation}
y_i = \ell_i(f_0),
\qquad i \in [1:m].
\end{equation}
For convenience, we summarize these data as
\begin{equation}
\yb = L(f_0) = [\ell_{1}(f_0);\ldots;\ell_{m}(f_0)] \in \R^m,
\end{equation}
where the linear map $L:\Fc \to\R^{m}$ is called the observation operator.
Relevant situations include the case where $\Fc$ is the space $\mathcal{C}(\Omega)$ of continuous functions on $\Omega$,
which is equipped with the uniform norm,
and the case where $\Fc$ is a Hilbert space~$\Hilbert$,
which is equipped with the norm derived from its inner product.
It is the latter case that is the focus of this work.
More precisely, after recalling some known results,
we concentrate on a reproducing kernel Hilbert space $\Hilbert$ of functions defined on $\Omega$,
so that the point evaluations at the $\xb_i$'s are indeed well-defined and continuous linear functionals on $\Hilbert$. 

\subsection{The model set.}
Without further information,
data by themselves are not sufficient to say anything meaningful about~$f_0$. For example, one could think of all ways to fit a univariate function through points 
$(x_1,y_1),\ldots,(x_m,y_m) \in \R^2$
if no restriction is imposed.
Thus, a model assumption for the functions of interest is needed. 
This assumption takes the form
\begin{equation}
f_0 \in \K,
\end{equation}
where the model set $\K$ translates an educated belief about the behavior of realistic functions $f_0$.
In Optimal Recovery,
the set $\K$ is often chosen to be a convex and symmetric subset of $\Fc$.
Here, our relevant modeling assumption is the one that occurs implicitly in Machine Learning,
namely that the functions of interest are well-approximated by suitable hypothesis classes.
In this work, 
we only consider hypothesis classes that are linear subspaces $V$ of $\Fc$.
Thus, given an approximation parameter $\epsilon > 0$ (the targeted approximation error),
our model set has the form
\begin{equation}
\label{modelset}
    \K := \{ f\in \Fc:\dist(f,V) \le \epsilon\},
\end{equation}
where $\dist(f,V) := \inf\{ \|f-v\|_{\Fc}, v \in V \}$.
In the case $\Fc = \Hilbert$ of a Hilbert space,
this model set reads
\begin{equation}
\label{modelsetH}
    \K = \{ f\in \Hilbert: \| f - P_V f \|_\Hilbert \le \epsilon\},
\end{equation}
where $P_Vf$ is the orthogonal projection of $f$ onto the subspace~$V$.
Such an approximability set was
put forward in \cite{binev2017data},
with motivation coming from parametric PDEs. 
 When working with this model,
 it is implicitly assumed that 
\begin{equation}
\label{ImpAssump}
    V \cap \ker(L) = \{0\},
\end{equation}
otherwise the existence of a nonzero $v \in V \cap \ker(L)$
would imply that each $f_t := f_0 + t v$, $t \in \R$,
is both data-consistent ($L(f_t) = \yb$)
and model-consistent ($f_t \in \K$),
leading to infinite worst-case error by letting $t \to \infty$.
By a dimension argument,
the assumption \eqref{ImpAssump} forces 
\begin{equation}\label{n<m}
n:= \dim(V) \le m,
\end{equation}
i.e., we must place ourselves in 
an underparametrized regime where there are less model parameters than datapoints.
To make sense of the overparametrized regime,
the model set \eqref{modelset} would need to be refined by adding some boundedness conditions,
see \cite{foucart2020instances} for results in this direction.

\subsection{Worst-case errors.}
With the model set in place, we now need to assess the performance of a learning/recovery map,
which is just a map taking data $\yb \in \R^m$ as input and returning an element $\hat{f} \in \Fc$ as output.
Given a model set $\K$, 
the local worst-case error of such a map $R: \R^m \to \Fc$ at $\yb \in \R^m$ is
\begin{equation}
\label{locerr}
{\rm err}^{\rm loc}_\K(L,R(\yb))
:= \sup_{f \in \K, L(f)=\yb} \|f - R(\yb)\|_\Fc.
\end{equation}
The global worst-case error is the worst local worst-case error over all $\yb \in \R^m$
that can be obtained by observing some $f \in \K$,
i.e.,
\begin{equation}
\label{gloerr}
{\rm err}^{\rm glo}_\K(L,R)
:= \sup_{f \in \K} \|f - R(L(f))\|_\Fc.
\end{equation}
A learning/recovery map $R: \R^m \to \Fc$ is called locally, respectively globally,
optimal if it minimizes the local, respectively global, worst-case error.
These definitions can be extended to handle not only the full recovery of $f_0$ 
but also the recovery of a quantity of interest $Q(f_0)$.
That is,
for a map $Q: \Fc \to Z$ from $\Fc$ into another normed space $Z$,
one would define e.g. the global worst-case error of the learning/recovery map $R: \R^m \to Z$ as
\begin{equation}
{\rm err}^{\rm glo}_{\K,Q}(L,R)
:= \sup_{f \in \K} \|Q(f) - R(L(f))\|_Z.
\end{equation}
Such a framework is pertinent even if we target the full recovery of $f_0$
but with performance evaluated in a norm $\sslash \cdot \sslash_\Fc$ different from the native norm $\|\cdot\|_\Fc$,
as we can consider $Q$ to be the identity map from $\Fc$ equipped with $\|\cdot\|_\Fc$ into $Z=\Fc$ equipped with $\sslash \cdot \sslash_\Fc$.

Perhaps counterintuitively,
dealing with the global setting is somewhat  easier than dealing with the local setting,
in the sense that globally optimal maps have been obtained in situations where locally optimal maps have not,
e.g. when $\Fc = \mathcal{C}(\Omega)$.
Accordingly, it is the local setting which is the focus of this work.

\subsection{Computation of local worst-case errors.}
When $\Fc = \Hilbert$ is a Hilbert space
and the approximability model \eqref{modelsetH} is selected,
determining the local worst-case error of a given map $R: \R^m \to \Hilbert$ at some $\yb$ involves solving 
\begin{equation}
\label{WCEinfdim}
\maximize{f \in \Hilbert} \|f-R(\yb)\|_\Hilbert
\quad \mbox{s.to } 
\begin{cases}
\|f-P_Vf\|_\Hilbert \le \epsilon, \\
L(f)=\yb.   
\end{cases}
\end{equation}
This is a nonconvex optimization program,
and as such does appear hard to solve at first sight.
However,
it is a quadratically constrained quadratic program,
hence it is possible to solve it exactly.
Although Gurobi~\cite{gurobi} now features direct capabilities to solve quadratically constrained quadratic programs,
we take the route of recasting \eqref{WCEinfdim} as a semidefinite program using the S-lemma \cite{polik2007survey}.
The solution of the recast program can then be obtained using an off-the-shelf semidefinite solver,
at least when the dimension $N = \dim(\Hilbert)$ of the Hilbert~$\Hilbert$ space is finite.
Precisely,
with $(h_1,\ldots,h_N)$ denoting an orthonormal basis for $\Hilbert$ chosen in such a way that $(h_1,\ldots,h_{N-m})$ is an orthonormal basis for $\ker(L)$
and with $H$ denoting the unitary map $x \in \R^{N-m} \mapsto \sum_{k=1}^{N-m} x_k h_k \in \ker(L)$,
local worst-case errors can be computed based on the following observation.


\begin{theorem}
\label{computation_wce}
The local worst-case error of a learning/recovery map $R: \R^m \to \Hilbert$ at $\yb \in \R^m$
under the model set \eqref{modelset}
can be expressed,
with $g := R(\yb)$, as
\begin{align}
\label{expression_wce}
    e_{\K}^{\rm loc}&(L,g) = \\
    \nonumber
    &
    \left[
    \|h-P_{\ker(L)^\perp}(g)\|_\Hilbert^2 
    + \| P_{\ker(L)}(g) \|_\Hilbert^2
    + c^\star
    \right]^{1/2},
\end{align}
where $h$ is the unique element in $\ker(L)^\perp$ satisfying $L(h)=\yb$
and $c^\star$ is the minimal value of the following program,
in which $w := P_{V^\perp}(h)$:
\begin{align}
\label{semidef_minimization}
    & \minimize{c,d \in \R} \quad c
    \quad \mbox{\rm s.to } d \ge 0 
    \quad \mbox{\rm and }\\
    \nonumber
    & \begin{bmatrix}
    H^*(d P_{V^\perp} - I_\Hilbert)H 
    & | & H^*(d w + P_{\ker(L)}(g))\\
    \hline
    (d w + P_{\ker(L)}(g))^*H
    & | & c + d (\|w\|_\Hilbert^2 - \epsilon^2)
    \end{bmatrix}
    \succeq 0.
\end{align} 
\end{theorem}

\begin{proof}
We first justify the claim that there exists a unique $h\in\ker(L)^\perp$ such that $L(h)=\yb\in\R^m$. 
To see this, define the linear 
map $\Tilde{L} : h \in \ker(L)^\perp \mapsto L(h) \in {\rm range}(L)$.
Since $\ker(\Tilde{L}) = \ker(L)\cap\ker(L)^\perp=\{0\}$,
the map $\Tilde{L}$ is injective. Therefore, we have $\dim({\rm range}(\Tilde{L}))
= \dim( \ker(L)^\perp )$,
which equals $N - \dim(\ker(L)) = \dim({\rm range}(L))$
by the rank-nullity theorem,
so the map $\Tilde{L}$ is also surjective.
Thus, the claim is justified by the fact that $\Tilde{L}$ is bijective.

Next, the squared local worst-case error \eqref{locerr}
at $g = R(\yb)$ is
\begin{align}
\label{wce^2}
\big[ {\rm err}_\mathcal{K}^{\rm loc}&(L,g) \big]^2 =\\
\nonumber
    & \sup_{f\in\Hilbert}
    \big\{ \|f-g\|_\Hilbert^{2} :  \|P_{V^\perp}f\|_\Hilbert^2\leq \epsilon^2, L(f)=\yb
    \big\}.
\end{align}
Decomposing $f$ and $g$ as $f=f'+f''$ and $g=g'+g''$
with $f',g'\in\ker(L)$ and $f'',g''\in\ker(L)^\perp$, 
the condition $L(f)=\yb$ reduces to $L(f'')=\yb$, i.e., $f''=h$ is uniquely determined.  
The condition $\|P_{V^\perp}f\|_\Hilbert^2\leq \epsilon^2$
then becomes $\|P_{V^\perp}f'+w\|_\Hilbert^2\leq\epsilon^2$.
As for the expression to maximize, it separates into
\begin{align}
\label{new_obj}
\|f-g\|_\Hilbert^2 & =    \|f''-g''\|_\Hilbert^2 + \|f'-g'\|_\Hilbert^2 \\ \nonumber
& = \|h-g''\|_\Hilbert^2 + \|g'\|_\Hilbert^{2} + \|f'\|_\Hilbert^2-2\langle f',g'\rangle.
\end{align}
Up to the additive constant $\|h-g''\|_\Hilbert^2 + \|g'\|_\Hilbert^{2}$,
the maximum in \eqref{wce^2} is now
\begin{align}
\label{new_optimization}
   & \sup_{f'\in\ker(L)} \|f'\|_\Hilbert^2-2\langle f',g'\rangle \; \mbox{s.to } \|P_{V^\perp}f'+w\|_\Hilbert^2\leq\epsilon^2\\
  \nonumber
 & = \inf_{c\in\R} c \; \mbox{s.to }  \|f'\|_\Hilbert^2\hspace{-1mm}-\hspace{-1mm}2\langle f',g'\rangle \hspace{-1mm}\leq \hspace{-1mm} c 
 \mbox{ when } \|P_{V^\perp}f'\hspace{-1mm}+\hspace{-1mm}w\|_\Hilbert^2 \hspace{-1mm} \leq \hspace{-1mm} \epsilon^2.
 \end{align}
Writing $f'=Hx$ with $x\in\R^{N-m}$,
this latter constraint reads
\begin{align}
\label{inf_contraints}
    & c-\big(\langle Hx,Hx\rangle-2\langle Hx,g'\rangle\big) \geq 0  \nonumber
    \end{align}
    whenever
    $$\epsilon^2-\big(\langle P_{V^\perp}Hx,P_{V^\perp}Hx \rangle +2\langle P_{V^\perp}Hx,w\rangle + \|w\|_\Hilbert^2\big)\geq 0.$$
By the S-lemma, see e.g. \cite{polik2007survey},  \eqref{inf_contraints} is equivalent to the existence of $d \ge 0$ such that
\begin{align}
    & c-\big(\langle Hx,Hx\rangle-2\langle Hx,g'\rangle\big) \geq   \\ \nonumber
    & d \, \big[\epsilon^2-\big(\langle P_{V^\perp}Hx,P_{V^\perp}Hx \rangle +2\langle P_{V^\perp}Hx,w\rangle+\|w\|_\Hilbert^2\big)\big]
\end{align}
for all $x \in \R^{N-m}$,
or in other words,
to the existence of $d \ge 0$ such that
\begin{align}
    & \big( d \langle x, (H^* P_{V^\perp} H)x \rangle - 
    \langle x, H^*Hx \rangle\big)
    \\ \nonumber 
    & + 2 \big( d \langle x,H^* w \rangle+ \langle x, H^* g' \rangle \big)  + c + d(\|w\|_\Hilbert^2 - \epsilon^2) \ge 0
\end{align}
for all $x \in \R^{N-m}$.
This constraint can be reformulated as a semidefinite constraint
\begin{equation}
    \begin{bmatrix}
    dH^*P_{V^\perp}H - H^*H 
    & | & d H^* w + H^* g'\\
    \hline
    (d H^* w + H^* g')^*
    & | & c + d (\|w\|_\Hilbert^2 - \epsilon^2)
    \end{bmatrix}
    \succeq 0.
\end{equation}
Keeping in mind that $g' = P_{\ker(L)}g$, this is the semidefinite constraint appearing in \eqref{semidef_minimization}.
Putting everything together,
we arrive at the expression for the local worst-case error announced in \eqref{expression_wce}.
\end{proof}

\subsection{Optimal learning/recovery map.}
Even though it is possible to compute the minimal worst-case error via \eqref{expression_wce}-\eqref{semidef_minimization},
optimizing over $g \in \Hilbert$ to produce the locally optimal recovery map would still require some work and would in fact be a major overkill.
Indeed, for our situation of interest,
some crucial work in this direction has been carried out in~\cite{binev2017data},
and we rely on it to derive the announced user-friendly formula for the optimal recovery map $R^{\rm opt}$.
Precisely, 
when $\Fc = \Hilbert$ is a (finite- or infinite-dimensional) Hilbert space and the model set $\K$ is given by \eqref{modelsetH},
it was shown in \cite{binev2017data}  that,
for any input $\yb \in \R^m$,
the output $R^{\rm opt}(\yb) \in \Hilbert$ is the solution $\hat{f}$ to the convex minimization program
\begin{equation}
\label{Hilbert_minimization}
\minimize{f \in \Hilbert} \|f - P_V f\|_\Hilbert
\qquad \mbox{subject to } L(f)=\yb.
\end{equation}
We generalize this result through Theorem \ref{ThmGenBinevEtAl} in the appendix.
It suffices to say for now that
the argument of \cite{binev2017data}, based on the original expression \eqref{locerr} of the worst-case error, exploits the fact that $\hat{f} - P_V\hat{f}$ is orthogonal not only to $V$ but also to $\ker(L)$.
Let us point out that $R^{\rm opt}(\yb) = \hat{f}$ is both data-consistent and model-consistent when $\yb=L(f_0)$ for some $f_0 \in \K$. 
It is also interesting to note that the optimal recovery map $R^{\rm opt}$ does not depend on the approximation parameter $\epsilon$.
This peculiarity disappears as soon as observation errors are taken into consideration,
see~\cite{ettehad2020instances}.

A computable expression for the minimal local error \eqref{locerr},
and in turn for the minimal global error \eqref{gloerr},
has also been given in~\cite{binev2017data}.
Without going into details,
we only want to mention that the
latter decouples as the product $\mu \times \epsilon$ of an indicator~$\mu$
of compatibility between model and datapoints,
which increases as the space~$V$ is enlarged,
and of the parameter~$\epsilon$ of approximability, 
which decreases as the space~$V$ is enlarged.
Thus, the choice of a space~$V$ yielding small minimal worst-case errors involves a trade-off on $n = \dim(V)$.
This trade-off is illustrated numerically in Subsection~\ref{CompTE}. 

Although the description given by of the optimal learning/recovery map is quite informative,
it fails to make apparent the fact the map $R^{\rm opt}$
is actually a linear map.
This fact can be seen from the theorem below,
which states that solving a minimization program for each $\yb \in \R^m$ 
is not needed to produce $R^{\rm opt}(\yb)$.
Indeed, one can obtain $R^{\rm opt}(\yb)$ by some linear algebra computations involving two matrices which are more or less directly available to the learner.
To define these matrices,
we need the Riesz representers $u_i \in \Hilbert$ of the linear functionals $\ell_{i} \in \Hilbert^*$,
which are characterized by 
$$
\ell_i(f) = \langle u_i, f \rangle
\qquad \mbox{for all }f \in \Hilbert.
$$
We also need a (not necessarily orthonormal) basis $(v_1,\ldots,v_n)$ for $V$.
The two matrices are the Gramian $\Gb \in \R^{m \times m}$ of $(u_1,\ldots,u_m)$ and the cross-Gramian $\Cb \in \R^{m \times n}$ of $(u_1,\ldots,u_m)$
and $(v_1,\ldots,v_n)$.
Their entries are given, for $i,i'\in[1:m]$ and $j\in[1:n]$, by
\begin{align}
\label{Gram}
        (\Gb)_{i,i'} &= \langle u_{i},u_{i'}\rangle = \ell_{i}(u_{i'}), \\
        \label{CrossGram}
        \Cb_{i,j} &= \, \langle u_{i},v_{j}\rangle \, = \ell_{i}(v_{j}).
\end{align}
The matrix $\Gb$ is positive definite and in particular invertible 
(linear independence of the $\ell_i$'s is assumed). 
The matrix $\Cb$ has full rank thanks to the assumption $V\cap\ker(L)=\{0\}$.
The result below
shows that the output of the optimal learning/recovery map does not have to lie in the space $V$ 
(the hypothesis class),
as opposed to the output of algorithms such as empirical risk minimizations. 

\begin{theorem}
\label{Explicit_OR}
The locally optimal learning/recovery map $R^{\rm opt}:~\R^m~\to~\Hilbert$
is given in closed form for each $\yb \in \R^m$ by
\begin{equation}\label{ORmap}
    R^{\rm opt}(\yb) = \sum_{i=1}^{m}a_{i}u_{i} + \sum_{j=1}^{n}b_{j}v_{j}, 
\end{equation}
where the coefficient vectors $\ab\in\R^{m}$ and $\bb\in\R^{n}$ are computed as 
\begin{align}
\label{coef_b}
 \bb & = (\Cb^\top \Gb^{-1} \Cb)^{-1} \Cb^\top \Gb^{-1} \yb,\\
 \label{coef_a}
 \ab & = \Gb^{-1}(\yb-\Cb\bb).
\end{align}
\end{theorem}

\begin{proof}
Let $\hat{f} = R^{\rm opt}(\yb)$ be the solution to \eqref{Hilbert_minimization}. 
We point out 
(as already mentioned or as a special case of \eqref{OrthProp}) that $\hat{f}-P_{V}\hat{f}$ is orthogonal to the space $\ker(L)$.
This property completely characterizes $\hat{f}$ as the element given by \eqref{ORmap}.
Indeed,
in view of $\ker(L)^\perp = \vspan\{ u_1,\ldots,u_m \}$, 
we have
\begin{equation}
    \hat{f} - P_{V}\hat{f} = \sum_{i=1}^{m} a_{i}u_{i} \qquad \mbox{for some } \ab \in\R^m. 
\end{equation}
Taking inner product with $v_{1},\dots,v_{n}$ leads to ${\bf 0}=\Cb^\top \ab$. 
Then, expanding $P_{V}\hat{f}$ on $(v_{1},\dots,v_{n})$, we obtain
\begin{equation}
    \hat{f} = \sum_{i=1}^{m}a_{i}u_{i} + \sum_{j=1}^{n}b_{j}v_{j} \qquad \mbox{for some } \bb\in\R^n. 
\end{equation}
Taking inner product with $u_{1},\dots,u_{m}$ leads to $\yb=\Gb\ab+\Cb\bb$ and in turn to $\Cb^\top \Gb^{-1}\yb = \Cb^\top \Gb^{-1}\Cb \bb$ after multiplying by $\Cb^\top \Gb^{-1}$. 
The latter yields the expression for $\bb$ given in \eqref{coef_b}, while the former yields the expression for $\ab$ given in \eqref{coef_a}.
\end{proof}


\section{Relation to Supervised Learning}
Supervised learning algorithms take data $\yb \in \R^m$ as input 
(while also being aware of the $\xb_i$'s)
and return functions $\hat{f} \in \Hilbert$ as outputs,
so they can be viewed as learning/recovery maps $R: \R^m \to \Hilbert$.
We examine below how some of them compare to the map $R^{\rm opt}$ from Theorem \ref{Explicit_OR}.

\subsection{Empirical risk minimizations.}
By design, the outputs $\hat{f}$ returned by these algorithms belong to a hypothesis space chosen in advance from the belief that it provides good approximants for real-life functions.
Since this implicit belief parallels
the explicit assumption expressed by the model set \eqref{modelset},
our Optimal Recovery algorithm and empirical risk minimization algorithms are directly comparable,
in that they both depend on a common approximation space/hypothesis class $V$.
With a loss function chosen as a $p$th power of an $\ell_p$-norm for $p \in [1,\infty]$,
empirical risk minimization algorithms consist in solving the convex optimization program
\begin{equation}
\label{ERM}
    \minimize{f \in \Hilbert}
    \|\yb - L(f)\|_p^p
    = \sum_{i=1}^m |y_i - \ell_i(f)|^p
    \quad \mbox{s.to }
    f \in V.
\end{equation}
In the case $p=2$ of the square loss, the solution actually reads
\begin{equation}\label{Rerm}
    R^{{\rm erm}_2}(\yb)
    = \sum_{j=1}^n \left( (\Cb^\top \Cb)^{-1} \Cb^\top \yb \right)_j v_j,
\end{equation}
where the matrix $\Cb \in \R^{m \times n}$ still represents the cross-Gramian introduced in \eqref{CrossGram}.

\subsection{Kernel regressions.}
Kernel regression algorithms usually operate in the setting of Reproducing Kernel Hilbert Spaces 
(see next section),
but they can be phrased for arbitrary Hilbert spaces, too.
For instance,
the traditional kernel ridge regression consists in solving the following convex optimization problem
\begin{equation} 
\label{KReg}
    \minimize{f\in \Hilbert} \sum_{i=1}^m (y_i - \ell_i(f))^2 + \gamma \|f\|_{\Hilbert}^2
\end{equation}
for some parameter $\gamma > 0$.
In the limit $\gamma \to 0$,
one obtains kernel ridgeless regression, 
which consists in solving the convex optimization problem 
\begin{equation}
\label{KRG_Min}
    \minimize{f\in \Hilbert} \|f\|_{\Hilbert}
    \qquad \mbox{s.to }  \ell_i(f)=y_i,
    \quad i \in [1:m].
\end{equation}
This algorithm fits the training data perfectly
and is also known to generalize well \cite{Liang2018JustIK}.

The crucial observation we wish to bring forward here is that kernel ridgeless regression,
although not designed with this intention,
is also an Optimal Recovery method.
Indeed, \eqref{KRG_Min} appears as the special case of the convex optimization program \eqref{Hilbert_minimization} with the choice $V=\{0\}$.
Using Theorem \ref{Explicit_OR},
we can retrieve in particular that kernel ridgeless regression is explicitly given by
\begin{equation}\label{ridgeless}
    R^{\rm ridgeless}(\yb) = \sum_{i=1}^m \left( \Gb^{-1} \yb \right)_i u_i.
\end{equation}
Incidentally, the latter can also be interpreted as the special case $V = \vspan\{u_1,\ldots,u_m\}$,
since $\hat{f} = R^{\rm ridgeless}(\yb)$ is a linear combination of the Riesz representers $u_1,\ldots,u_m$ that satisfy the observation constraint $L(\hat{f})=\yb$.
In fact, there are  more choices for $V$ that leads to kernel ridgeless regression, 
as revealed below.

\begin{theorem}
\label{arbitraryV}
If the space is
$V = \vspan\{ u_i, i \in I\}$
for some subset $I$ of $[1:m]$, then the locally optimal recovery map \eqref{Hilbert_minimization} reduces to kernel ridgeless regression independently of $I$.
\end{theorem}

\begin{proof}
Let $V=\vspan\{u_{i},i\in I\}$ for some $I \subseteq [1:m]$ 
and let $\hat{f}$ be the output of kernel ridgeless regression.
According to the proof of Theorem \ref{Explicit_OR}, 
to prove that $\hat{f}$ is the solution to \eqref{Hilbert_minimization},
we have to verify that $\hat{f} - P_V \hat{f} \in \ker(L)^\perp$.
Since we already know that $\hat{f} = \hat{f} - P_{\{0\}} \hat{f} \in \ker(L)^\perp$
(recall that kernel ridgeless regression is \eqref{Hilbert_minimization} with $\{0\}$ in place of~$V$),
it remains to check that $P_V \hat{f} \in \ker(L)^\perp$.
This simply follows from 
$P_V \hat{f} \hspace{-.5mm}\in \hspace{-.5mm} \vspan\{u_{i},i\in I\}
\hspace{-.5mm} \subseteq \hspace{-.5mm} \vspan\{u_1,\ldots,u_m\} 
\hspace{-.5mm} = \hspace{-.5mm}
\ker(L)^\perp$.
\end{proof}

\subsection{Spline models.}
From an Optimal Recovery point-of-view,
the success of \eqref{KRG_Min} can be surprising because it seems to use only data and no model assumption.
In fact, the model assumption occurs in the objective function being minimized.
Procedure \eqref{KRG_Min} favors  data-consistent functions which are themselves small.
If one preferred to favor data-consistent functions which have small derivatives,
one would instead consider, say,  
the program
\begin{equation}
\label{Spline_Min}
    \minimize{f\in W_{2}^{k}[0,1]} \|f^{(k)}\|_{L_{2}} \qquad \mbox{s.to } f(x_{i})=y_{i}, \quad i\in[1:m],
\end{equation}
with optimization variable $f$ in the Sobolev space $W_{2}^{k}[0,1]$. 
As it turns out,
this procedure coincides with the Optimal Recovery method that minimizes the worst-case error over the model set given by $\mathcal{K}=\{f\in W_{2}^{k}[0,1]:\|f^{(k)}\|_{L_{2}}\leq1\}$
and its solution is known explicitly \cite{de1963best}.
With $k=2$
(where one tries to minimize the strain energy of a curve constrained to pass through a prescribed set of points),
the solution is a cubic spline,
see \cite{wahba1990spline} for details.
For multivariate functions,
the solutions to problems akin to \eqref{Spline_Min} are also known explicitly:
they are thin plate splines \cite{duchon1977splines}.
More generally,
minimum-(semi)norm interpolation problems are what define the concept of abstract splines \cite{de1981convergence}.

{\bf Remark.}
When observation error is present,
exact interpolation conditions should not be enforced,
so it is natural to subsitutute \eqref{Hilbert_minimization} by a regularized problem similar to \eqref{KReg} but with $\|f-P_V f\|_\Hilbert^2$ acting as a reguralizer instead of $\|f\|_\Hilbert^2$.
This has already been proposed in \cite{li2007generalized} under the name
Generalized Regularized Least-Squares,
of course with a different motivation than Optimal Recovery. In fact,
a more general regularized problem where $\|f-P_V f\|_\Hilbert^2$ gives way to a squared seminorm also appeared in \cite{vito2004some}.
Such inverse-problem inspired techniques have been applied to statistical learning e.g. in 
\cite{Vito2005LearningFE}, which studied the consistency properties of the associated estimators.

\section{Optimal Recovery in Reproducing Kernel Hilbert Spaces (RKHS)}

We consider in this section the case where $\Fc = \Hilbert$ is a Hilbert space of functions defined on a domain $\Omega \subseteq \R^d$
for which point evaluations are continuous linear functionals.
In other words, we consider a reproducing kernel Hilbert space $\Hilbert_K$,
where $K: \Omega \times \Omega \to \R$ denotes the kernel characterized,
for any $\xb \in \Omega$, by
\begin{equation}
f(\xb) = \langle K(\xb,\cdot), f \rangle
\qquad \mbox{for all } f \in \Hilbert_K.
\end{equation}
In this way, the Riesz representers of
points evaluations at $\xb_i$'s
take the form $u_i = K(\xb_i,\cdot)$.
Thus, the Gramian of \eqref{Gram} has entries
\begin{equation}
\Gb_{i,i'} \hspace{-1mm}=\hspace{-1mm} \langle K(\xb_i,\cdot), K(\xb_{i'},\cdot) \rangle \hspace{-1mm}=\hspace{-1mm} K(\xb_{i'},\xb_i),
\; i,i' \in [1:m].
\end{equation}
As for the cross-Gramian of \eqref{CrossGram},
it has entries
\begin{equation}
\Cb_{i,j} = v_j(\xb_i),
\qquad i \in [1:m],
\; j \in [1:n],
\end{equation}
where $(v_1,\ldots,v_n)$ represents a basis for the space $V$.
Some possible choices of $K$ and $V$ are discussed below.

\subsection{Choosing the kernel.}
A kernel that is widely used in many learning problems 
is the Gaussian kernel given,
for some parameter $\sigma>0$, by
\begin{equation}
    K(\xb,\xb')=\exp\left(
    -\frac{\|\xb-\xb'\|^{2}}{2\sigma^{2}}
    \right),
    \qquad \xb,\xb' \in \R^d.
\end{equation}
The associated infinite-dimensional Hilbert space,
explicitly characterized in \cite{minh2010gaussiankernel},
has orthonormal basis $\{ \phi_\alpha, \alpha \in \mathbb{N}_0^d\}$, where
\begin{equation}
\label{ONBGK}
    \phi_{\alpha}(\xb) \hspace{-1mm}=\hspace{-1mm} \sqrt{\frac{(1/\sigma^{2})^{\alpha_1+\cdots+\alpha_d}}{\alpha_1! \cdots \alpha_d!}} \exp\hspace{-1mm}\left(-\frac{\|\xb\|^{2}}{2 \sigma^{2}}\right)\hspace{-1mm}
    x_1^{\alpha_1} \cdots x_d^{\alpha_d}.
\end{equation}


\subsection{Choosing the approximation space.} 
Since a learning/recovery procedure uses both data and model (maybe implicitly),
its performance depends on the interaction between the two.
In Optimal Recovery,
and subsequently in Information-Based Complexity \cite{traub2003information},
it is often assumed that the model is fixed and that the user has the ability to choose evaluation points in a favorable way.
From another angle,
one can view the evaluation points as being fixed but the model could be chosen accordingly. 
For the applicability of Theorem \ref{Explicit_OR},
it is perfectly fine to select an approximation space $V$ depending on $\xb_1,\ldots,\xb_m$,
so long as it does not depend on $y_1,\ldots,y_m$.
Thus, one possible choice for the approximation space consists of 
$V = \vspan\{ K(\xb_i,\cdot), i \in I\}$
for some subset $I \subseteq [1:m]$. 
However, we have seen in Theorem \ref{arbitraryV} that such a choice invariably leads to kernel ridgeless regression.
Another choice for the approximation space is inspired by linear regression, which uses the space $\vspan\{1,x_1,\ldots,x_d\}$.
We do not consider this space verbatim, because its elements 
(or any polynomial function, for that matter, see \cite{minh2010gaussiankernel}) do not belong to the reproducing kernel Hilbert space with Gaussian kernel. Instead, we modify it slightly by multiplying with a decreasing exponential
and by allowing for degrees $k$ higher than one, so as to consider the space
\begin{equation}
\label{V_taylorfeature}
    V = \vspan\{ \phi_\alpha, \alpha_1+\cdots+\alpha_d \le k\},
\end{equation}
which has dimension $n = \binom{d+k}{d}$. We ignore the coefficients of $\phi_\alpha$ in numerical experiments, which has no effects on the test error. 
These $\phi_\alpha$'s are the so-called `Taylor features'
used in approximation of the Gaussian kernel \cite{cotter2011explicit}. 


\section{Experimental Validation}\label{sec:simul}

\subsection{Comparison of worst-case errors}
\label{CompWCE}

We first compare worst-case errors for the optimal recovery map (OR) described in Theorem \ref{Explicit_OR}
and for empirical risk minimizations  defined in \eqref{ERM}. 
They are only considered with $p=1$ (ERM1) and $p=2$ (ERM2). 
The algorithms
OR, ERM1, and ERM2 all operate with a specific space $V$ (as a hypothesis class), so direct comparisons can be made by selecting the same $V$ for all these algorithms.
According to Theorem \ref{computation_wce}, when $\Hilbert$ is a finite-dimensional Hilbert space, 
the computation of their worst-case errors is performed by semidefinite programming.
Here, we restrict ourselves to the case where $V$ is a $n$-dimensional subspace of $\Hilbert = \ell_{2}^N$, with $n=20$ and $N = 200$. 
The $m=50$ linear observations are randomly generated. Figure \ref{wce_plot}(a) confirms that OR yields the smallest worst-case errors,
hints at a quasi-linear dependence of the worst-case errors on the approximability parameter $\epsilon$,
and suggests that ERM2 yields smaller worst-case errors than ERM1.

In contrast,
Figure \ref{wce_plot}(b) suggests that ERM1 yields smaller worst-case errors than ERM2 when the standard empirical risk minimization \eqref{ERM}
is enhanced by replacing the overdemanding constraint $f \in V$ by the constraint $f \in \Kc$,
i.e., $\|f-P_Vf\|_\Hilbert \le \epsilon$.
Although the performances are now very close for all algorithms,
it has to be noted that in this case running ERM1 and ERM2 requires an a priori knowledge of $\epsilon$
while running OR does not.

\begin{figure*}[!t]
    \centering
    \begin{subfigure}[t]{0.5\textwidth}
        \centering
        \includegraphics[height=1.8in]{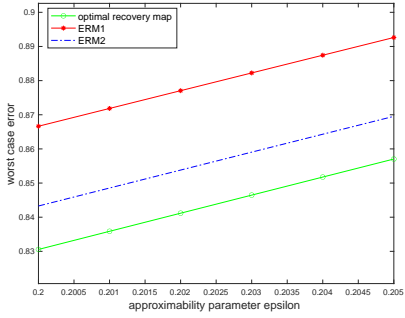}
        \caption{ ERM1 and ERM2 are produced with constraint $f\in V$.}
    \end{subfigure}%
    ~ 
    \begin{subfigure}[t]{0.5\textwidth}
        \centering
        \includegraphics[height=1.8in]{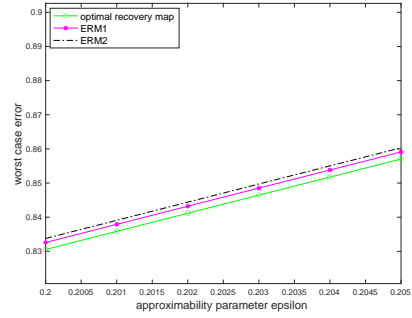}
        \caption{ERM1 and ERM2 are produced with constraint $f\in \Kc$}
    \end{subfigure}
    \caption{Optimal Recovery and Empirical Risk Minimization maps with $p=1$ and $p=2$.}
    \label{wce_plot}
\end{figure*}

\subsection{Test errors for non-IID data}\label{CompTE}
In this subsection, we implement the optimal recovery map on two real-world regression datasets, namely Years Prediction and Energy Use, both available on UCI Machine Learning Repository. 

We focus on the reproducing kernel Hilbert space $\Hilbert_K$ associated with Gaussian kernel throughout this experiment.
The space $V$ is spanned by a subset of Taylor features of order $k=1$, see \eqref{V_taylorfeature}, so that $\dim(V)$ goes up to $d+1$, where $d$ is the number of features in the datasets.
To choose the optimal kernel width, we conduct a grid search. Furthermore, to make the data non-IID, we sort both datasets according to their $5$-th feature in a descending order and then select the top $70\%$ as the training set and the bottom $30\%$ as the test set. 
Recall by Theorem \ref{Explicit_OR} that the optimal recovery map depends on the Hilbert space $\Hc_K$ and the subspace $V$. Therefore, it is natural to compare it to kernel ridgeless regression \eqref{ridgeless} (in $\Hc_K$) and Taylor features regression \eqref{Rerm} (in $V$).

The test error comparison is presented in Figure \ref{testerror_plot}. 
Due to the size of Years Prediction dataset, we do not perform kernel ridgeless regression on the full dataset, so we randomly subsample a $5000$ subset of the data and repeat the experiment for $40$ Monte Carlo simulations to average out the randomness. 
Therefore, error bars are presented in Figure \ref{testerror_plot}(a) to show the statistical significance. 
We observe that the optimal recovery map shows promising performance on both datasets. On Years Prediction dataset, Optimal Recovery outperforms kernel ridgeless regression for all $\dim(V)$.
On Energy Use dataset, 
it outperforms kernel ridgeless regression after $\dim(V)=2$. 
Also, Taylor features regression in the space $V$ is consistently inferior to the optimal recovery map. 
The U-shape Optimal Recovery curve in Figure \ref{testerror_plot}(a) demonstrates the trade-off between the compatibility indicator $\mu$ and the approximability parameter $\epsilon$. 

\begin{figure*}[!t]
    \centering
    \begin{subfigure}[t]{0.5\textwidth}
        \centering
        \includegraphics[height=1.8in]{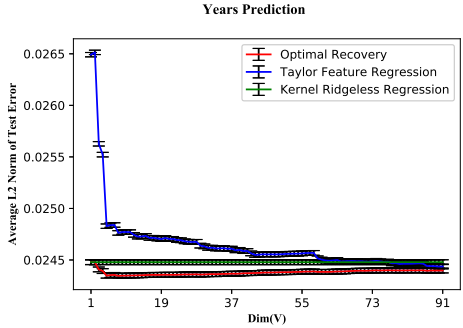}
        \caption{Test error comparison on Years Prediction}
    \end{subfigure}%
    ~ 
    \begin{subfigure}[t]{0.5\textwidth}
        \centering
        \includegraphics[height=1.8in]{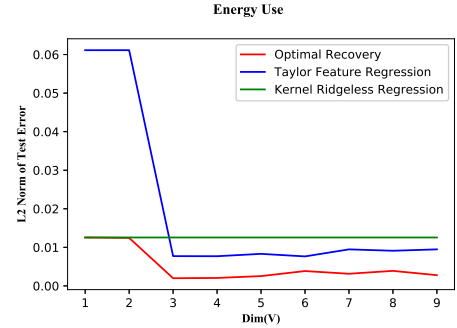}
        \caption{Test error comparison on Energy Use}
    \end{subfigure}
    \caption{Optimal Recovery and two  benchmark regression algorithms on two benchmark datasets.}
    \label{testerror_plot}
\end{figure*}

\section{Conclusion}
Generalization guarantees in Statistical Learning are based on the postulate of IID data, the pertinence of which is not guaranteed in all learning environments.
In this work, we considered the regression problem (with non-random data) in Hilbert spaces from an Optimal Recovery point-of-view, where the learner aims at minimizing the worst-case error. 
We first formulated a semidefinite program for calculating the worst-case error of any recovery map in finite-dimensional Hilbert spaces. 
Then, we provided a closed-form expression for optimal recovery map in the case where the hypothesis class $V$ is a linear subspace of any Hilbert space.
The formula coincides with kernel ridgeless regression when $V=\{0\}$ in a reproducing kernel Hilbert space. 
Our numerical experiments showed that, when $\dim(V)>0$, Optimal Recovery has the potential to outperform kernel ridgeless regression in the test mean squared error.

Our main focus was to provide an algorithmic perspective to Optimal Recovery, whose theory was initiated in the 70's-80's. 
Our findings revealed interesting connections with current Machine Learning methods. 
There are many directions to consider in the future, including:
\begin{enumerate}
\item[(i)] learning the hypothesis space $V$ from the data (instead of incorporating domain knowledge); 
\item[(ii)] developing Optimal Recovery with noise/error in the observations;
\item[(iii)] studying the overparametrized regime $\dim(V)>m$;
\item[(iv)] investigating the case where the hypothesis class $V$ is not a linear space.
\end{enumerate}




%


\section*{Acknowledgment}

C.~L. and Y.~W. are supported by the Texas A\&M Triads for Transformation (T3) Program. S.~F. is partially supported by NSF grants DMS-1622134 and DMS-1664803. S.~F and S.~S. also acknowledge NSF grant CCF-1934904.

\ifCLASSOPTIONcaptionsoff
  \newpage
\fi

\section*{Appendix}

Below,
we generalize the result of \cite{binev2017data} in two directions. 
For the first direction, 
instead of assuming that the target function $f_0$ itself is well approximated by elements of a set $V$,
we assume that it is some linear transform $T$ applied to $f_0$ that is well approximated.
This translates in the modification \eqref{ModifK} of the approximability set.
The novelty occurs not for invertible transforms,
but for noninvertible ones
(e.g. when $T$ represents a derivative, as in \eqref{Spline_Min}).
For the second direction,
instead of attempting to recover $f_0$ in full,
we assume that we only need to estimate a quantity 
$Q(f_0)$ depending on $f_0$,
such as its integral.
Although we focus on the extreme situations where $Q$ is the identity or where $Q$ is a linear functional,
the case of an arbitrary linear map $Q$ is covered.
Leaving the introduction of the transform $T$ aside,
one useful consequence of the result below is that
knowledge of (a basis for) the space~$V$ is not needed,
since only the values of the $\ell_i(v_j)$'s and $Q(v_j)$'s are required to form $(Q \circ R^{\rm opt})(\yb)$.

\begin{theorem}
\label{ThmGenBinevEtAl}
Let $\mathcal{F},\Hilbert,\mathcal{Z}$ be three normed spaces,
$\Hilbert$ being a Hilbert space,
and let $V$ be a subspace of $\Hilbert$.
Consider a linear quantity of interest $Q: \mathcal{F} \to \mathcal{Z}$
and a linear map $T: \mathcal{F} \to \Hilbert$.
For $\yb \in \R^m$,
define $R^{\rm opt}(\yb) \in \mathcal{F}$ as a solution to
\begin{equation}
\label{MinProgWithT}
    \minimize{f\in \mathcal{F}} \|Tf - P_V(Tf)\|_{\Hilbert} 
    \qquad \mbox{s.to }
    L(f) = \yb.
\end{equation}
Then the learning/recovery map $Q \circ R^{\rm opt}: \R^m \to \mathcal{Z}$
is locally optimal over the model set 
\begin{equation}
\label{ModifK}
    \K = \{ f \in \mathcal{F}: 
    {\rm dist}(Tf,V) \le \epsilon \}
\end{equation}
in the sense that, for any $z \in \mathcal{Z}$,
\begin{multline}
\label{OptWithT}
    \sup_{f \in \K, L(f) = \yb}
    \|Q(f) - Q \circ R^{\rm opt}(f)\|_{\mathcal{Z}}\\
    \le 
    \sup_{f \in \K, L(f) = \yb}
    \|Q(f) - z \|_{\mathcal{Z}}.
\end{multline}
\end{theorem}

\begin{proof}
Let us introduce the compatibility indicator
\begin{equation}
\label{DefMu}
    \mu := \sup_{u \in \ker(L) \setminus \{0\}}
    \frac{\|Q(u)\|_\mathcal{Z}}{{\rm dist}(Tu,V)}.
\end{equation}
Given $\yb \in \R^m$,
let $\hat{f} = R^{\rm opt}(\yb)$ denote the solution to \eqref{MinProgWithT}.
We shall establish \eqref{OptWithT} by showing on the one hand that 
\begin{equation}
\label{WithT1}
    \sup_{\substack{f \in \K\\ L(f) = \yb}}\hspace{-1mm}
    \|Q(f) \hspace{-1mm}-\hspace{-1mm} Q(\hat{f})\|_{\mathcal{Z}}
    \le \mu \big[ \epsilon^2 \hspace{-1mm}-\hspace{-1mm} \|T\hat{f} \hspace{-1mm}-\hspace{-1mm} P_V(T \hat{f})\|_\Hc^2 \big]^{1/2}
\end{equation}
and on the other hand that,
for any $z \in \mathcal{Z}$.
\begin{equation}
\label{WithT2}
    \sup_{\substack{f \in \K\\ L(f) = \yb}}\hspace{-1mm}
    \|Q(f) \hspace{-1mm}-\hspace{-1mm} z \|_{\mathcal{Z}}
    \ge \mu \big[ \epsilon^2 \hspace{-1mm}-\hspace{-1mm} \|T\hat{f} \hspace{-1mm}-\hspace{-1mm} P_V(T \hat{f})\|_\Hc^2 \big]^{1/2}.
\end{equation}
Let us start with \eqref{WithT1}.
Considering an arbitrary $u \in \ker(L)$,
notice that the quadratic expression $t \in \R$ given by
\begin{align}
\nonumber
  \|&T(\hat{f} + t u) \hspace{-1mm} - \hspace{-1mm} P_V(T(\hat{f} + t u))\|_{\Hilbert}^2  
  = \|T\hat{f}  \hspace{-1mm} - \hspace{-1mm} P_V(T\hat{f} )\|_{\Hilbert}^2 \\
 & + 2 t \langle T\hat{f}  \hspace{-1mm} - \hspace{-1mm} P_V(T\hat{f} ),Tu \hspace{-1mm} - \hspace{-1mm} P_V(Tu) \rangle + \mathcal{O}(t^2)
\end{align}
is miminized at the point $t=0$.
This forces the linear term $\langle T\hat{f}  - P_V(T\hat{f} ),Tu - P_V(Tu) \rangle$ to vanish, in other words 
\begin{equation}
    \label{OrthProp}
    \langle T\hat{f}  - P_V(T\hat{f} ),Tu \rangle = 0
    \qquad \mbox{for any }
    u \in \ker(L).
\end{equation}
Now, considering $f \in \K$ such that $L(f) = \yb$
written as $f = \hat{f} + u$ for some $u \in \ker(L)$,
the fact that $f \in \K$ reads
\begin{align}
\nonumber
    \epsilon^2 & \ge
    \|T(\hat{f} + u) \hspace{-1mm} - \hspace{-1mm} P_V(T(\hat{f} +  u))\|_{\Hilbert}^2\\
    & = 
    \|T\hat{f} \hspace{-1mm} - \hspace{-1mm} P_V(T\hat{f})\|_{\Hilbert}^2
    + \|Tu \hspace{-1mm} - \hspace{-1mm} P_V(Tu)\|_{\Hilbert}^2.
\end{align}
Rearranging the latter gives
\begin{equation}
    {\rm dist}(Tu,V) \le \big[ \epsilon^2 \hspace{-1mm}-\hspace{-1mm} \|T\hat{f} \hspace{-1mm}-\hspace{-1mm} P_V(T \hat{f})\|_\Hc^2 \big]^{1/2}.
\end{equation}
It remains to take the definition  \eqref{DefMu}
into account in order to bound $\|Q(f) - Q(\hat{f}) \|_{\mathcal{Z}} = \|Q(u) \|_{\mathcal{Z}}$ and arrive at \eqref{WithT1}.

Turning to \eqref{WithT2}, we consider $u \in \ker(L)$ such that
\begin{align}
\label{Ac1}
    \|Q(u)\|_{\mathcal{Z}} 
    & = \mu \, {\rm dist}(Tu, V),\\
    \label{Ac2}
    \|Tu \hspace{-1mm} - \hspace{-1mm} P_V(Tu)\|_\Hilbert
    & = \big[ \epsilon^2 \hspace{-1mm}-\hspace{-1mm} \|T\hat{f} \hspace{-1mm}-\hspace{-1mm} P_V(T \hat{f})\|_\Hc^2 \big]^{1/2}.
\end{align}
It is clear that $f^\pm := \hat{f} \pm u$ both satisfy $L(f^\pm) = \yb$,
while $f^\pm \in \K$ follows from
\begin{align}
\nonumber
    \|T f^\pm \hspace{-1mm}&-\hspace{-1mm} P_V(T f^\pm)\|_\Hc^2
    = \|(T\hat{f} \hspace{-1mm}-\hspace{-1mm} P_V(T \hat{f}) ) \pm (T u \hspace{-1mm}-\hspace{-1mm} P_V(T u) ) \|_\Hc^2\\
    & = \|T\hat{f} \hspace{-1mm}-\hspace{-1mm} P_V(T \hat{f}) \|^2 + \| Tu \hspace{-1mm}-\hspace{-1mm} P_V(T u)  \|_\Hc^2
    = \epsilon^2.
\end{align}
Therefore, for any $z \in \mathcal{Z}$,
\begin{align}
\nonumber
    \sup_{\substack{f \in \K\\ L(f) = \yb}}\hspace{-1mm}
    \|Q(f) \hspace{-1mm}&-\hspace{-1mm} z \|_{\mathcal{Z}}
     \ge \max\{ \|Q(f^+) \hspace{-1mm}-\hspace{-1mm} z \|_{\mathcal{Z}},
    \|Q(f^-) \hspace{-1mm}-\hspace{-1mm} z \|_{\mathcal{Z}}\}  \\
    \nonumber
    & \ge \frac{1}{2} \big( 
    \|Q(f^+) \hspace{-1mm}-\hspace{-1mm} z \|_{\mathcal{Z}} + 
    \|Q(f^-) \hspace{-1mm}-\hspace{-1mm} z \|_{\mathcal{Z}}\}
    \big)\\
    & \ge \frac{1}{2}  \| Q(f^+ - f^-) \|_{\mathcal{Z}}
    = \| Q(u) \|_{\mathcal{Z}}.
\end{align}
Taking \eqref{Ac1} and \eqref{Ac2} into account finishes to prove \eqref{WithT2}.
\end{proof}

\bibliographystyle{./IEEEtran}
\bibliography{refs.bib}

\begin{thebibliography}{10}
\providecommand{\url}[1]{#1}
\csname url@samestyle\endcsname
\providecommand{\newblock}{\relax}
\providecommand{\bibinfo}[2]{#2}
\providecommand{\BIBentrySTDinterwordspacing}{\spaceskip=0pt\relax}
\providecommand{\BIBentryALTinterwordstretchfactor}{4}
\providecommand{\BIBentryALTinterwordspacing}{\spaceskip=\fontdimen2\font plus
\BIBentryALTinterwordstretchfactor\fontdimen3\font minus
  \fontdimen4\font\relax}
\providecommand{\BIBforeignlanguage}[2]{{%
\expandafter\ifx\csname l@#1\endcsname\relax
\typeout{** WARNING: IEEEtran.bst: No hyphenation pattern has been}%
\typeout{** loaded for the language `#1'. Using the pattern for}%
\typeout{** the default language instead.}%
\else
\language=\csname l@#1\endcsname
\fi
#2}}
\providecommand{\BIBdecl}{\relax}
\BIBdecl

\bibitem{vapnik1999overview}
V.~N. Vapnik, ``An overview of statistical learning theory,'' \emph{IEEE
  Transactions on Neural Networks}, vol.~10, no.~5, pp. 988--999, 1999.

\bibitem{MicRiv}
C.~A. Micchelli and T.~J. Rivlin, ``A survey of optimal recovery,'' in
  \emph{Optimal Estimation in Approximation Theory}.\hskip 1em plus 0.5em minus
  0.4em\relax Springer, 1977, pp. 1--54.

\bibitem{MicRiv2}
------, ``Lectures on optimal recovery,'' in \emph{Numerical Analysis Lancaster
  1984}.\hskip 1em plus 0.5em minus 0.4em\relax Springer, 1985, pp. 21--93.

\bibitem{de1963best}
C.~De~Boor, ``Best approximation properties of spline functions of odd
  degree,'' \emph{Journal of Mathematics and Mechanics}, pp. 747--749, 1963.

\bibitem{duchon1977splines}
J.~Duchon, ``Splines minimizing rotation-invariant semi-norms in {S}obolev
  spaces,'' in \emph{Constructive Theory of Functions of Several
  Variables}.\hskip 1em plus 0.5em minus 0.4em\relax Springer, 1977, pp.
  85--100.

\bibitem{de1977computational}
C.~De~Boor, ``Computational aspects of optimal recovery,'' in \emph{Optimal
  Estimation in Approximation Theory}.\hskip 1em plus 0.5em minus 0.4em\relax
  Springer, 1977, pp. 69--91.

\bibitem{packel1988linear}
E.~W. Packel, ``Do linear problems have linear optimal algorithms?'' \emph{SIAM
  Review}, vol.~30, no.~3, pp. 388--403, 1988.

\bibitem{belkin2018understand}
M.~Belkin, S.~Ma, and S.~Mandal, ``To understand deep learning we need to
  understand kernel learning,'' in \emph{International Conference on Machine
  Learning}, 2018, pp. 541--549.

\bibitem{rakhlin2019consistency}
A.~Rakhlin and X.~Zhai, ``Consistency of interpolation with {L}aplace kernels
  is a high-dimensional phenomenon,'' in \emph{Conference on Learning Theory},
  2019, pp. 2595--2623.

\bibitem{Liang2018JustIK}
T.~Liang and A.~Rakhlin, ``Just interpolate: Kernel "ridgeless" regression can
  generalize,'' \emph{Annals of Statistics}, 2019.

\bibitem{hazan2016introduction}
E.~Hazan, ``Introduction to online convex optimization,'' \emph{Foundations and
  Trends in Optimization}, vol.~2, no. 3-4, pp. 157--325, 2016.

\bibitem{zhao2018federated}
Y.~Zhao, M.~Li, L.~Lai, N.~Suda, D.~Civin, and V.~Chandra, ``Federated learning
  with non-iid data,'' \emph{arXiv preprint arXiv:1806.00582}, 2018.

\bibitem{plaskota1996noisy}
L.~Plaskota, \emph{Noisy information and computational complexity}.\hskip 1em
  plus 0.5em minus 0.4em\relax Cambridge University Press, 1996, vol.~95,
  no.~55.

\bibitem{ettehad2020instances}
M.~Ettehad and S.~Foucart, ``Instances of computational optimal recovery:
  dealing with observation errors,'' \emph{arXiv preprint arXiv:2004.00192},
  2020.

\bibitem{beck2007regularization}
A.~Beck and Y.~C. Eldar, ``Regularization in regression with bounded noise: A
  chebyshev center approach,'' \emph{SIAM Journal on Matrix Analysis and
  Applications}, vol.~29, no.~2, pp. 606--625, 2007.

\bibitem{binev2017data}
P.~Binev, A.~Cohen, W.~Dahmen, R.~DeVore, G.~Petrova, and P.~Wojtaszczyk,
  ``Data assimilation in reduced modeling,'' \emph{SIAM/ASA Journal on
  Uncertainty Quantification}, vol.~5, no.~1, pp. 1--29, 2017.

\bibitem{foucart2020instances}
S.~Foucart, ``Instances of computational optimal recovery: refined
  approximability models,'' \emph{arXiv preprint arXiv:2004.00195}, 2020.

\bibitem{gurobi}
\BIBentryALTinterwordspacing
{Gurobi Optimization, LLC}, ``Gurobi optimizer reference manual,'' 2020.
  [Online]. Available: \url{http://www.gurobi.com}
\BIBentrySTDinterwordspacing

\bibitem{polik2007survey}
I.~P{\'o}lik and T.~Terlaky, ``A survey of the {S}-lemma,'' \emph{SIAM Review},
  vol.~49, no.~3, pp. 371--418, 2007.

\bibitem{wahba1990spline}
\BIBentryALTinterwordspacing
G.~Wahba, \emph{Spline Models for Observational Data}.\hskip 1em plus 0.5em
  minus 0.4em\relax Society for Industrial and Applied Mathematics, 1990.
  [Online]. Available:
  \url{https://epubs.siam.org/doi/abs/10.1137/1.9781611970128}
\BIBentrySTDinterwordspacing

\bibitem{de1981convergence}
C.~De~Boor, ``Convergence of abstract splines,'' \emph{Journal of Approximation
  Theory}, vol.~31, no.~1, pp. 80--89, 1981.

\bibitem{li2007generalized}
W.~Li, K.-H. Lee, and K.-S. Leung, ``Generalized regularized least-squares
  learning with predefined features in a {H}ilbert space,'' in \emph{Advances
  in Neural Information Processing Systems}, 2007, pp. 881--888.

\bibitem{vito2004some}
E.~De~Vito, L.~Rosasco, A.~Caponnetto, M.~Piana, and A.~Verri, ``Some
  properties of regularized kernel methods,'' \emph{J. Mach. Learn. Res.},
  vol.~5, p. 1363–1390, Dec. 2004.

\bibitem{Vito2005LearningFE}
E.~D. Vito, L.~Rosasco, A.~Caponnetto, U.~Giovannini, and F.~Odone, ``Learning
  from examples as an inverse problem,'' \emph{J. Mach. Learn. Res.}, vol.~6,
  pp. 883--904, 2005.

\bibitem{minh2010gaussiankernel}
H.~Minh, ``Some properties of {G}aussian reproducing kernel {H}ilbert spaces
  and their implications for function approximation and learning theory,''
  \emph{Constructive Approximation}, vol.~32, p. 307–338, 2010.

\bibitem{traub2003information}
J.~F. Traub, \emph{Information-{B}ased {C}omplexity}.\hskip 1em plus 0.5em
  minus 0.4em\relax John Wiley and Sons Ltd., 2003.

\bibitem{cotter2011explicit}
A.~Cotter, J.~Keshet, and N.~Srebro, ``Explicit approximations of the
  {G}aussian kernel,'' \emph{arXiv preprint arXiv:1109.4603}, 2011.

\end{thebibliography}

\end{document}